%% file: arxiv_topk.tex
\documentclass[11pt]{article}

\usepackage[letterpaper,margin=1in]{geometry}
\usepackage[parfill]{parskip}

\usepackage{authblk}

\usepackage[toc,page,header]{appendix}

\usepackage[utf8]{inputenc} 
\usepackage[T1]{fontenc}    
\usepackage[colorlinks,citecolor=blue,urlcolor=blue,linkcolor=blue,linktocpage=true]{hyperref}       
\usepackage{url}            
\usepackage{booktabs}       
\usepackage{amsfonts}       
\usepackage{nicefrac}       
\usepackage{microtype}      
\usepackage{xcolor}         
\usepackage{comment}
\usepackage{thmtools}
\label{key}

\usepackage{amsmath}

\usepackage{epsfig}
\usepackage{graphicx}
\usepackage{wrapfig}
\usepackage{subfigure}
\usepackage{multirow}
\usepackage{graphicx}
\usepackage{caption}
\usepackage{array}
\usepackage{bbm}
\usepackage{color}
\usepackage{ulem}
\usepackage{enumerate}
\usepackage{enumitem}
\setlist{leftmargin=10mm}
\usepackage{mathtools}
\usepackage{amsmath,amssymb,amsthm,bm}
\usepackage{amsfonts,graphicx}
\usepackage{mathrsfs}
\usepackage{algorithmic,algorithm}

\usepackage[authoryear]{natbib}

\def\algbackskip{\hskip-\ALG@thistlm}

\def\pr{\mathrm{Pr}}
\def\E{\mathbb{E}}
\def\P{\mathbb{P}}
\def\argmax{Argmax}

\def \E{\mathbb{E}}

\def\R{\mathbb{R}}
\def\cA{\mathcal{A}}

\def\cD{\mathcal{D}}

\def\cO{\mathcal{O}}

\def\cM{\mathcal{M}}
\def\cN{\mathcal{N}}

\def\cR{\mathcal{R}}

\usepackage{cleveref}
\newtheorem{theorem}{Theorem}
\newtheorem{lemma}[theorem]{Lemma}

\newtheorem{definition}[theorem]{Definition}
\newtheorem{example}[theorem]{Example}
\newtheorem*{remark}{Remark}

\newtheorem{remark-star}{Remark}
\newtheorem{remark-star-1}{Remark}

\newtheorem{corollary}[theorem]{Corollary}
\newtheorem{proposition}[theorem]{Proposition}
\newtheorem*{proof-sketch}{Proof Sketch}
\usepackage{booktabs}

\newcommand{\explainup}[2]{\overset{\mathclap{\underset{\downarrow}{#2}}}{#1}}

\author[1]{Yuqing Zhu}
\author[1]{Yu-Xiang Wang}
\affil[1]{Computer Science Department, UC Santa Barbara}
\title{Adaptive Private-K-Selection with Adaptive K and Application to Multi-label PATE}

\begin{document}


\maketitle

\begin{abstract}
	
We provide an end-to-end Renyi DP based-framework for differentially private top-$k$ selection.  Unlike previous approaches, which require a data-independent choice on $k$, we propose to privately release a data-dependent choice of $k$ such that the gap between $k$-th and the $(k+1)$st ``quality'' is large. This is achieved by a novel application of the Report-Noisy-Max.
Not only does this eliminate one hyperparameter, the adaptive choice of $k$ also certifies the stability of the top-$k$ indices in the unordered set so we can release them using a variant of propose-test-release (PTR) without adding noise. We show that our construction improves the privacy-utility trade-offs compared to the previous top-$k$ selection algorithms theoretically and empirically. Additionally, we apply our algorithm to ``Private Aggregation of Teacher Ensembles (PATE)'' in \textit{multi-label} classification tasks with a large number of labels and show that it leads to significant performance gains.  
	
\end{abstract}

\input{introduction.tex}
\input{preliminary.tex}
\input{method.tex}

\input{experiment.tex}
\vspace{-2mm}
\section{Conclusion}
\vspace{-2mm}
To conclude, we develop an efficient private top-$k$ algorithm with an end-to-end RDP analysis. We  generalize  the Report-Noisy-Max algorithm, the \textit{propose-test-release} framework and the \textit{distance-to-instability} framework with Gaussian noise and formal RDP analysis. In the downstream task, we show our algorithms improve the performance of the model-agnostic framework with multi-label classification. We hope this work will spark more practical applications of private selection algorithms.

\section*{Acknowledgments}
\vspace{-4mm}
The work was partially supported by NSF Award \# 2048091, a Google Research Scholar Award and a gift from NEC Labs. Yuqing was partially supported by a Google PhD Fellowship. The authors would like to thank the anonymous reviewers for helpful feedback, as well as Jinshuo Dong and Ryan Rogers to helpful discussions related to the exponential mechanism.

\bibliographystyle{plainnat}
\bibliography{DP}



\newpage
\appendix
\onecolumn
\input{new_appendix.tex}

\end{document}

%% file: introduction.tex
\section{Introduction}
\vspace{-2mm}


The private top-$k$ selection problem~\citep{durfee2019, carvalho2020differentially, dwork2018differentially, hardt2013beyond} is one of the most fundamental problems in privacy-preserving data analysis. For example, it is a key component in several more complicated differentially private tasks, including private model selection, heavy hitter estimation and dimension reduction.  More recently, the private selection algorithm (Report-Noisy-Max) is combined with the ``Private Aggregation of Teacher Ensembles (PATE)''~\citep{papernot2017,papernot2018,bassily2018model} to build a knowledge transfer framework for model agnostic private learning.  In this work, we focus on the counting problem. Given a finite set of candidates and associated counts for each candidate, our goal is to design practical differential private algorithms that can return the \textit{unordered} top-$k$ candidates.
 
Unlike previous approaches on private top-$k$ selections, which assume $k$ is predetermined and data-independent, we consider choosing $k$ adapted to the data itself.  
Why would an adaptive choice of $k$ be preferred? We give two reasons.
First, a data-dependent choice of $k$ captures the informative structure of the dataset. To illustrate this matter, consider the following sorted sequence of utilities
$$ [100, 100, \explainup{99}{\text{Example A:} k = 3 },99,98,...,98, \explainup{54}{\text{Example B: } k =20},53,53,52,50,...],$$ 
In Example A, the analyst choose $k=3$ but all scores up to the $19$th are of nearly the same utility, and the total utility should not differ much among any three within the top 19; similarly in Example B, the index set of the top $20$ does not reveal the large gap between the 19th and 20th, and the selection is somewhat unfair to both the top 19 and the 21st onwards to include the 20th. We argue that is more natural to choose $k=19$ in a data-dependent way.  The second, and more technical, reason is because it is substantially cheaper in terms of the privacy-budget to accurately release the top 19 in this example than either top 3 or top 20. Even if the end goal is to get top 20,  the large-gap structure can be leveraged so that one can achieve better accuracy (at the same privacy cost) by first releasing the top 19 and then choose an arbitrary index to pad to $k=20$.


To formalize these intuitions, we propose an elegant two-step procedure that first privately select the most-appropriate $k$ using a variant of Report-Noisy-Max, then use a \textit{propose-test-release} (\textit{PTR}) approach  \citep{dwork2009differential} to privately release the set of indices of the top $k$ candidates.  When the chosen gap is large, then the \textit{PTR}  algorithm adds no noise at all with high probability.  We also propose an extension of our approach to handle the case when there is a target $k$ of interest. Empirical and analytical results demonstrate the utility improvements compared to the state-of-the-arts, which encouragingly suggests that using the \textit{PTR} as a drop-in replacement could make top-$k$ selection-based algorithms more practical in downstream tasks.

Our contributions are four-folds:
\begin{enumerate}
	\vspace{-2mm}
	\itemsep0em
	\item We introduce a new differentially private, efficient algorithm for the top-$k$ selection problem with an end-to-end RDP analysis.
	\item We show that our algorithms improve over the existing state-of-the-art private top-$k$ selection algorithms with a formal utility analysis and an empirical comparison on real sensitive datasets.
	\item We extend the Report-Noisy-Max algorithm and the propose-test-release framework with Gaussian noise distribution and provide RDP analysis for two variants. Empirically, we show that two variants are more advantageous than their Laplace counterparts under compositions.
	\item Our algorithms enable private model-agnostic learning with multi-label classification with a practical privacy-utility tradeoff.
\end{enumerate}




\noindent\textbf{Related work and novelty.} The private-$k$-selection has seen a growing interest in the machine learning and differential privacy community~\citep{chaudhuri2014large, mcsherry2009differentially, banerjee2012price, durfee2019, carvalho2020differentially}.

Notably, the iterative peeling approach that composes $k$ exponential mechanisms (EM) has been shown to be minimax optimal.  \citet{durfee2019} shows that adding Gumbel noise and reports the top-k in one-shot is equivalent to using the exponential mechanism with peeling.   
Recent work~\citep{qiao2021oneshot} adapts Report-Noisy-Max to select the top-$k$ elements and achieves $(\epsilon, \delta)$-DP with a noise level of $\tilde{O}(\sqrt{k}/\epsilon)$.  We note that the released ``top-$k$'' indices are ordered in the above work, and therefore the dependence on $k$ is unavoidable in the $\epsilon$ term. 
The focus of this work is to privately release an unordered set of the top-$k$ indices and get rid of the dependence in $k$. 

We are the first to consider privately choosing hyperparameter $k$ and leverage the large-gap with these choices for adapting to the favorable structure in each input.  
The closest to us is perhaps \citep{carvalho2020differentially} in which the algorithm also leverages the large-gap information to avoid the dependence in $k$ by combining the sparse vector technique (SVT) and the \textit{distance to instability framework}, however, it still requires a fixed $k$ and a ``crude'' superset with cardinality $\tilde{k}$.  Our approach is simpler and more flexible.  The utility comparison section demonstrates that our algorithm achieves better utility over \citet{carvalho2020differentially, durfee2019} under the same ``unknown-domain'' setting.

Technically, our method builds upon the PTR-framework and Report-Noisy-Max with extensions tailored to our problem of interest. Our RDP analysis of RNM with other noise-adding mechanisms (e.g., Gaussian noise)  is based on the proof technique of \citep{zhu2020improving} for analyzing SVT. Our approach may strike the readers as being very simple, but we emphasize that ``constant matters in differential privacy'' and the simplicity is precisely the reason why our method admits a tight privacy analysis. In our humble opinion all fundamental problems in DP should admit simple solutions and we are glad to have found one for private-k-selection. 





%% file: preliminary.tex
\vspace{-2mm}
\section{Preliminary}
\vspace{-2mm}
In this work, we study the problem of \textit{differential private top-k selection} in the user-counting setting\footnote{All our results also apply to the more general setting of selection among an arbitrary set of low-sensitivity functions, but the user-counting setting allows a tighter constant and had been the setting existing literature on this problem focuses on. }. Consider a dataset of $n$ users is defined as $D=\{x_1,  ..., x_n\}$. We say that two dataset $D$ and $D'$ are neighboring, if they differ in any one user's data, e.g. $D = D' \cup \{x_i\}$.  Assume a candidate set contains $m$ candidates $\{1, ..., m\}$. 
  We consider the setting where a user can vote $1$ for an arbitrary number of candidates, i.e. unrestricted sensitivity.  One example for the unrestricted setting would be calculating the top-$k$ popular places that users have visited.  
 We use $x_{j,i}$ to denote the voting of user $i$, e.g., $x_{j,i}=1$ indicates the $i$-th user vote $1$ for the $j$-th candidate.  Let $h_j(D)\in \mathbb{N}$ denote the number of users that have element $j \in [m]$, i.e. $h_j(D) = \sum_{i=1}^n \mathbb{I} \{x_{j,i}=1 \}$ (we will drop $D$ when it is clear from context).  
 
  We then sort the counts and denote $h_{(1)}(D) \geq ... \geq h_{(m)}(D)$ as the sorted counts where  $i_{(1)}, ..., i_{(m)} \in [m]$ are the corresponding candidates. 
  Our goal is to design a differentially private mechanism that outputs  the \textit{unordered} set $\{i_{(1)}, ..., i_{(k)}\}$ which $k$ is chosen adaptively to private data itself. Formally, the algorithm returns a $m$-dim indicator  $\mathbb{I}(D) \in \{0,1\}^m$, where $\mathbb{I}_j = 1$ if $j \in \{i_{(1)}, ..., i_{(k)}\}$, otherwise $\mathbb{I}_j=0$.

  \noindent \textbf{Symbols and notations.} Throughout the paper, we will use the standard notations for probability, e.g., $\pr[\cdot]$ for probability, $p[\cdot]$ for density, $\mathbb{E}$ for expectation. $\epsilon, \delta$ are reserved for privacy loss parameters, and $\alpha$ the order of Renyi DP.
We now introduce the definition of differenital privacy.

\begin{definition}[Differential privacy~\citep{dwork2006calibrating}]
	A randomized algorithm $\cM$ is $(\epsilon, \delta)$-differential private if for neighbring dataset $D$ and $D'$ and all possible outcome sets $\cO \subseteq Range(\cM)$:
	\[\pr[\cM(D) \in \cO] \leq e^\epsilon \pr[\cM(D')\in \cO] + \delta\]
\end{definition}
Differential Privacy ensures that an adversary could not reliably infer whether one particular individual is in the dataset or not, even with arbitrary side-information.
\begin{definition}[Renyi DP \citep{mironov2017renyi}]
	We say a randomized algorithm $\cM$ is $(\alpha, \epsilon_\cM(\alpha))$-RDP with order $\alpha \geq 1$ if for neighboring datasets $D, D'$
	\begin{align*}
	&\mathbb{D}_{\alpha}(\cM(D)||   \cM(D')):=\\
	& \frac{1}{\alpha-1}\log \mathbb{E}_{o \sim \cM(D')}\bigg[ \bigg( \frac{\pr[\cM(D)=o]}{\pr[\cM(D')=o]}\bigg)^\alpha \bigg]\leq \epsilon_\cM(\alpha).
	\end{align*}
\end{definition}
At the limit of $\alpha \to \infty$, RDP reduces to $(\epsilon, 0)$-DP.  If $\epsilon_\cM(\alpha) \leq {\rho \alpha}$ for all $\alpha$, then we say that the algorithm satisfies $\rho$-zCDP~\citep{bun2016concentrated}.
This  more-fine-grained description  often allows for  a tighter $(\epsilon, \delta)$-DP over compositions compared to the strong composition theorem in~\citet{kairouz2015composition}.  Therefore, we choose to formulate  the privacy guarantee of our algorithms under the RDP framework.
Here, we introduce two properties of RDP that we will use.
\begin{lemma}[Adaptive composition]
	$\epsilon_{(\cM_1, \cM_2)} = \epsilon_{\cM_1}(\cdot) + \epsilon_{\cM_2}(\cdot)$.
\end{lemma}
\begin{lemma}[From RDP to DP] If a randomized algorithm $\cM$ satisfies $(\alpha,\epsilon(\alpha))$-RDP, then $\cM$ also satisfies $(\epsilon(\alpha)+\frac{\log(1/\delta)}{\alpha-1},\delta)$-DP for any $\delta \in (0,1)$. \label{lem: rdp2dp}
\end{lemma}
Next,  we will introduce the notion of approximate RDP, which generalizes   approximate zCDP~\citep{bun2016concentrated}.
\begin{definition}[Approximate RDP / zCDP]\label{def: approx}
	We say a randomized algorithm $\cM$ is $\delta$-approximately-$(\alpha, \epsilon_\cM(\alpha))$-RDP with order $\alpha \geq 1$, if for all neighboring dataset $D$ and $D'$, there exist events $E$ (depending on $\cM(D)$ )and $E'$ (depending on $\cM(D')$) such that $\pr[E]\geq 1-\delta$ and $\pr[E']\geq 1-\delta$, and $\forall \alpha\geq 1$, we have
	$	\mathbb{D}_{\alpha}(\cM(D)|E||   \cM(D')|E')\leq  \epsilon_\cM(\alpha)$. When $\epsilon_\cM(\alpha)\leq \alpha \rho$ for $\alpha\geq 1$ then $\cM$ satisfies $\delta$-approximate $\rho$-zCDP.
\end{definition}
This notion preserves all the properties as approximate zCDP~\citep{bun2016concentrated}. The reason for rephrasing it under the RDP framework is that some of our proposed algorithms satisfy  tighter RDP guarantees (compared to its zCDP version) while others satisfy RDP conditioning on certain high probability events. Similar conversion and composition rules of approximate-RDP are deferred to the appendix.

Many differentially private algorithms, including output perturbation, enable DP working by calibrating noise using the sensitivity. We start by defining the local and global sensitivity.
\begin{definition}[Local / Global sensitivity]
	The local sensitivity of $f$ with the dataset $D$ is defined  as 
	$LS_f(D) = \sup_{D'\sim D} ||f(D) - f(D')||$ and the global sensitivity of $f$ is $GS_f:=\sup_{D} LS_f(D).$
\end{definition}
The norm $||\cdot||$ could be any vector $\ell_p$ norm, and the choice on $\ell_p$ depends on which kind of noise we use, e.g., we calibrate Gaussian noise for Gaussian mechanism using $\ell_2$ norm.

\noindent\textbf{Motivation of an adaptive k.}
Recent work\citep{carvalho2020differentially, durfee2019,gillenwater2022joint} make use of structures in the top-$k$ counts, showing that large gaps improve the performance of the private top-$k$ mechanisms. This leads to one natural question ---  can't we just set a $k$, such that there exists a large gap between the $k$
and the $(k+1)$th vote?  Indeed, exploiting such the largest eigengap information is already a standard heuristic in selecting the number of principal components in PCA.  
Our result shows that if there is a large gap between $k$-th and the $(k+1)$th, we can return the top $k$ set with only two times privacy budget instead of $k$ times.  Moreover, even if want a pre-defined $k$ and  there is a large gap at $(k-3)$, 
then we can release the top $(k-3)$ with two times the budget then release the remaining using the exponential mechanism with $3$ times the budget. Our motivation is to adapt to these large-margin structures.

%% file: method.tex
\vspace{-2mm}
\section{Methods}
\vspace{-2mm}

We now present our main algorithms for data-adaptive top-k selection.  Section~\ref{sec:select_k} describes a simple algorithm that privately selects parameter  $k \in [m]$ such that it maximizes the gap $h_{(j)} - h_{(j+1)}$.   
Section~\ref{sec: ptr} presents a propose-test-release style algorithm called \textsc{StableTopK}. It first privately selects $k \in [m]$ such that it maximizes gap,  then  releases the top-$k$ index set whenever the gap at the chosen $k$ is large.   Section~\ref{sec: fix_k} demonstrates how  \textsc{StableTopK} can be used for the fixed $k$ setting, where  the algorithm takes $k$ as an input and is required to return exactly $k$ indices. 
 
\subsection{Choose a $k$ privately}\label{sec:select_k}
Recall that the goal is to choose $k$  that approximately maximizes the gap $h_{(k)} - h_{(k+1)}$.
Our idea of choosing $k$ uses off-the-shelf differentially private (Top-1) selection algorithms. Any private selection algorithm will work, but for simplicity we focus on the exponential mechanism \citep{mcsherry2007mechanism}, which is recently shown to admit a Report-Noisy-Max style implementation and a more refined privacy analysis via a ``Bounded Range'' property \citep{durfee2019}.  

The pseudo-code is given in Algorithm~\ref{alg: large_gap}. Readers may notice that it also takes a regularizer $r$.  The choice of $r$ can be arbitrary and can be used to encode additional \textit{public} information that the data analyst supplies  such  as hard constraints or priors that describe the \textit{ball park} of interest.  


\begin{algorithm}[t]
	\caption{Regularized Large Gap}
	\label{alg: large_gap}
	\begin{algorithmic}[1]
		\STATE{ \textbf{Input} Histogram $h$, regularizer $r: [m-1]\rightarrow \R$; DP parameter $\epsilon$.}
		\STATE Sort $h$ into a descending order $h_{(1)},h_{(2)}...,h_{(m)}$.
			\STATE{\textbf{Return}}\\
			$\argmax_{j \in [m-1]} \Big\{h_{(j)} - h_{(j+1)} + r(j)  + \mathrm{Gumbel}(\frac{2}{\epsilon})\Big\}.$
	\end{algorithmic}
\end{algorithm}

\begin{proposition}
	Algorithm~\ref{alg: large_gap} satisfies (pure)-$\epsilon$-DP, $\epsilon^2/8$-zCDP and and $(\alpha,\epsilon(\alpha))$-RDP with 
	{\small
	$$
	\epsilon(\alpha) := \min\left\{ \frac{\alpha \epsilon^2}{8},  \frac{1}{\alpha-1}\log(\frac{\mathrm{sinh}(\alpha \epsilon)-\mathrm{sinh}((\alpha-1)\epsilon)}{\mathrm{sinh}(\epsilon)})\right\}.
	$$
}
\end{proposition}
\begin{proof}
	As we are applying the exponential mechanism off-the-shelf, it suffices to analyze the sensitivity of the utility function $u(j) := h_{(j)} - h_{(j+1)} + r(j) $. Let $u,u'$ be the utility function of two neighboring dataset (with histograms $h,h'$). For any $j$
	\begin{align*}
		|u(j) - u'(j)| &=  |(h_{(j)} - h_{(j+1)}) -(h'_{(j)}- h'_{(j+1)})  |\leq 1.
	\end{align*}
The inequality can be seen by  discussing the two cases:``adding'' and ``removing'' separately. 
If we add one data point, it may only increase $h_{(j)}$ and $h_{(j+1)}$ by $1$. Similarly if we remove one data point it may only decrease $h_{(j)}$ and $h_{(j+1)}$ by $1$. In both cases, the change of the gap is at most $1$.  The pure-DP bound follows from \citet{mcsherry2007mechanism}, the zCDP bound follows from \citet[Lemma 17]{cesar2021bounding} and the RDP bound is due to \citet[Lemma 4]{bun2016concentrated}.
\end{proof}

Algorithm~\ref{alg: large_gap} is exponentially more likely to return a $k$ that has a larger gap than a $k$ that has a small gap. In our experiments, we find that the tighter zCDP analysis gives EM an advantage over other alternatives including the exponential noise and Laplace noise versions of RNM \citep{ding2021permute}.  For this reason, discussion of these other selection procedures are given in the appendix.

\textbf{Gaussian-RNM.} One may ask a natural question whether one can use more concentrated noise such as Gaussian noise to instantiate RNM.  Using the techniques from \citet{zhu2020improving}, we prove the following theorem about such generalized RNMs.
\begin{theorem}\label{thm: rnm_k}
	Let $\cM_g$ denote any noise-adding mechanism that satisfies $\epsilon_g(\alpha)$-RDP for a scalar function $f$ with global sensitivity $2$.  
	Assume Report Noisy Max adds the same magnitude of noise to each coordinate, then the algorithm obeys
	$\epsilon_\alpha(\cM(D)||\cM(D')) \leq \epsilon_g(\alpha) +\frac{\log m}{\alpha-1}$.
\end{theorem}

In particular, we introduce RNM-Gaussian as an alternative to RNM-Laplace with Gaussian noise.
\begin{corollary}[RNM-Gaussian]\label{coro_gau}
	RNM-Gaussian (the second line in Algorithm~\ref{alg: main_ada}) with Gaussian noise $\cN(0, \sigma^2)$ satisfies $(\frac{2\alpha}{\sigma^2}+\frac{\log m}{\alpha -1})$-RDP. 
\end{corollary}
We defer the comparison between RNM-variants in the appendix, suggesting that RNM-Gaussian is better than RNM-Laplace in certain regime (e.g., $m$ is not too large). However, RNM-Gumbel will dominate both of them over compositions.

\textbf{How to handle unknown domain / unlimited domain?} In TopK selection problems, it is usually desirable to be able to handle an unbounded $m$ in an unknown domain \citep{durfee2019,carvalho2020differentially}. Our method handles it naturally by taking the regularizer $r$ to be  a constraint that restricts our chooses to $j\in\{1,2,...,\bar{k}\}$ with an arbitrary $\bar{k}\ll m$. The issue of candidates moving inside and outside the top $\bar{k}$ is naturally handled by the selection of a stable $k$ within $\{1,2,...,\bar{k}\}$.  This simultaneously improves the RDP bound and the utility bound for RNM-Gaussian by replacing $m$ with $\bar{k}$.

\vspace{-2mm}
\subsection{Stable Top-$k$ selection with an adaptive $k$}\label{sec: ptr}
\vspace{-2mm}

Once $k$ is determined,  the next step is to privately release the top $k$ index set.  Different from existing methods that select the top $k$ by iteratively calling exponential mechanisms for $k$ times, we propose a new approach that release the unordered indices of the top $k$ at one shot using a propose-test-release (PTR)~\citep{dwork2009differential} style algorithm. The query of interest is the indicator vector  $\mathbb{I}_k(D)\in\{0, 1\}^m$ satisfying 
$$
[\mathbb{I}_k(D)]_j = \begin{cases}
1 &\text{ if } j \in  \text{TopK}\\
0 & \text{ otherwise.}
\end{cases}
$$
The indicator has a global L2 sensitivity of $\sqrt{2k}$,  as there are at most $k$ positions are $1$ in $\mathbb{I}(D)$ and $\mathbb{I}(D')$. It could appear to be a silly idea to apply Gaussian mechanism, because a naive application would require adding noise with scale $\approx \cN(0, \sqrt{2k}I_{m})$, rendering an almost useless release. 
Luckily, the problem happens to be one where the global sensitivity is way too conservative, and one can get away with adding a much smaller noise in a typical dataset, as the following lemma shows.

\begin{lemma}[Local sensitivity of the gap]
	Denote $q_k(D):= h_{(k)}(D) - h_{(k+1)}(D)$ as the gap between the $k$-th and the $k+1$-th largest count.  The local $\ell_2$ sensitivity of $q_k$ is $0$ if
	$q_k(D) > 1$.
\end{lemma}
\begin{proof}
Fix $k$. If we are adding, then it could increase $h_{(k+1)}(D)$ by at most $1$ and may not decrease  $h_{(k)}(D)$ . If we are removing, then it could decrease $h_{(k)}(D)$ by at most $1$ and may not increase $h_{(k+1)}(D)$. In either case, if  $q_k(D)>1$, it implies that $h_{{(k+1)}}(D') < h_{{(k)}}(D')$, thus the set of the top $k$ indices remains unchanged.
\end{proof}

Using the PTR approach, if we differentially privately test that the local sensitivity is indeed $0$, then we can get away with returning $\mathbb{I}(D)$ \textit{as is} without adding any noise. Notably, this approach avoids composition over $k$ rounds and could lead to orders of magnitude improvements over the iterative EM baseline when $k$ is large. A pseudocode of our proposed mechanism is given in Algorithm~\ref{alg: main_ada}. 

\begin{algorithm}[t]
	\vspace{-1mm}
	\caption{\textsc{StableTopK}: Private $k$ selection with an adaptive chosen $k$}
	\label{alg: main_ada}
	\begin{algorithmic}[1]
		\STATE \textbf{Input} Histogram $h$ and approximate zCDP budget parameters $\delta_t, \rho$.
		\STATE   Set $k$ by invoking Algorithm~\ref{alg: large_gap}  with $\epsilon=2\sqrt{\rho}$ (and arbitrary $r$).
		\STATE Set $q_k = h_{(k)} - h_{(k+1)}$ and $\sigma = \sqrt{1/\rho}$.
		\STATE Construct a high-probability lower bound\\
		$\hat{q}_k=  \max\{1, q_k\} + \cN(0, \sigma^2) -\sigma\sqrt{2\log(1/\delta_t)}.$
		\IF {$\hat{q}_k > 1$}
		\STATE  \textbf{Return} $i_{(1)}, ... , i_{(k)}$
		\ELSE
		\STATE  \textbf{Return} $\perp$.
		\ENDIF
%
	\end{algorithmic}
	\vspace{-1mm}
\end{algorithm}

\begin{theorem}
	Algorithm~\ref{alg: main_ada} satisfies $\delta_t$-approximated-$\rho$-zCDP and $(\rho + \sqrt{2\rho\log(1/\delta)}, \delta + \delta_t)$-DP for any $\delta\geq 0$.  Moreover, if the chosen $k$ satisfies that $q_k > 1 + 2\sqrt{\frac{2\log(1/\delta_t)}{\rho}}$, then the algorithm returns the correct top-$k$ set with probability $1-\delta_t$. 
\end{theorem}
\begin{proof}
	The mechanism is a composition of Algorithm~\ref{alg: large_gap} (by the choice of parameter, it satisfies $\rho/2$-zCDP) and an application of PTR which is shown to satisfy 
	$\delta_t$-approximate $\rho/2$-zCDP in Lemma~\ref{lem: ptr} in the appendix. The stated result is obtained by the composition of approximate zCDP and its conversion to $(\epsilon,\delta)$-DP. Finally, the utility statement follows straightforwardly from the standard subgaussian tail bound.
\end{proof}

\textbf{Utility comparison.} The theorem shows that our algorithm returns the correct Top-$k$ index with high probability if the gap $q_k $ is $O(\sqrt{\frac{2\log(1/\delta_t)}{\rho}})$.  In comparison, the iterative EM algorithm, or its limited domain (LD) variant \citep{durfee2019} requires the gap to be on the order of  $O(\sqrt{\frac{k}{\rho}}\log(1/\delta_t))$ --- a factor of $\sqrt{k\log(1/\delta_t)}$ worse than our results.  
Comparing to the Top Stable procedure (TS)~\citep{carvalho2020differentially}, which is similar to our method, but uses SVT instead of EM for selection; under the same condition (by Theorem 4.1 in their paper) TS requires the gap to be $\log(1/\delta_t) / \sqrt{\rho}$, which is a factor of $\sqrt{\log(1/\delta_t)}$ larger than our results.

\textbf{Connection to distance to instability framework}
Our algorithm has a nice connection with the distance to instability framework \citep{thakurta2013differentially}.  Similar to the idea of using gap information to upper bound the local sensitivity, we can define the \texttt{Dist2instability} function to be $\max\{0,h_{(k)}(D) - h_{(k+1)}(D)-1\}$ and test whether it is $0$ using Laplace mechanism. 
Our PTR-Gaussian algorithm can be thought of as an extension of the distance to instability framework with Gaussian noise, which is of independent interest.

\textbf{Why not smooth sensitivity?}
A popular alternative to PTR for such tasks of data-adaptive DP algorithm is the \textit{smooth sensitivity} framework ~\citep{nissim2007smooth}, which requires constructing an exponentially smoothed upper bound of the local sensitivity and add noise that satisfy certain  ``dilation'' and ``shift'' properties. Our problem does have an efficient smooth sensitivity calculation, however, we find that the ``dilation'' and ``shift'' properties of typical noise distributions (including more recent ones such as those proposed in\citet{bun2019average}) deteriorate exponentially as dimensionality gets large; making it infeasible for releasing an extremely high-dimensional vector in $\{0,1\}^m$. 

\subsection{Stable private $k$-selection with a fixed $k$}\label{sec: fix_k}
In many scenarios, $k$ is a parameter chosen by the data analyst who expect the algorithm to return exactly $k$ elements. In this situation, there might not be a large gap at $k$. In this section, we show that one can still benefit from a large gap in this setting if there exists one in the the neighborhood of the chosen $k$.

 We introduce \textsc{StableTopK} with a fixed $k$ (Algorithm~\ref{alg: fix_k}) which takes as input a histogram $h$, parameter $k$, regularizer parameter $\lambda$, and approximate zCDP parameter $\delta_t$ and $\rho$.
 
 \begin{algorithm}[t]
 	\vspace{-1mm}
 	\caption{ StableTopK with fixed $k$: Private Top-$k$ selection with a fixed $k$ input}
 	\label{alg: fix_k}
 	\begin{algorithmic}[1]
 		\STATE{ \textbf{Input} Histogram $h$, parameter $k$, regularizer weight $\lambda$,approx zCDP parameter $\delta_t, \rho$.}
 		\STATE Set  $r(j) =  -\lambda|j-k|$.
 		\STATE Set $\epsilon_{EM} = 2\sqrt{\rho}$.
 		\STATE   Set $S$ as the output of  Algorithm~\ref{alg: main_ada}, instantiated with ($h, \delta_t, \rho/2$) and regularizer $r$.
 		\STATE{\textbf{if}} $S=\perp$, \textbf{Return} result of Top-$k$ EM on $h$ with total pure-DP budget $\epsilon_{EM}$.
 		\STATE \textbf{if} $\tilde{k}=k$, \textbf{Return} $S$
 		\STATE \textbf{elif} $\tilde{k}>k$, \textbf{Return} result of Top-$k$ EM on $h_{i_{(1)}}, ..., h_{i_{(\tilde{k})}}$ with budget $\epsilon_{EM}$.
 		\STATE \textbf{else Return} $\{i_{1}, ..., i_{\tilde{k}}\} \cup$ result of Top-$(k-\tilde{k})$ EM on $h$ with budget $\epsilon_{EM}$.
 	\end{algorithmic}
 	\vspace{-1mm}
 \end{algorithm}
 
 Ideally, we hope to find a $\tilde{k}$ such that we see a sudden drop at the $\tilde{k}$-th position and $\tilde{k}$ is closed to the input $k$. Therefore, we introduce a regularizer term $\lambda|j-k|$ in Step 2.
  Then we apply PTR-Gaussian (Algorithm~\ref{alg: ptr_gau}) to privately release the top-$\tilde{k}$ elements. If $\tilde{k} <k$, we can optionally use exponential mechanism (\citep{mcsherry2007mechanism}) to privately select top-$(k-\tilde{k})$ elements in a peeling manner. Similarly, if $\tilde{k}>k$, we can apply exponential mechanism to select top-$k$ elements from the shrinked $h_{i_{(1)}, ..., i_{(\tilde{k})}}$ histogram.
 
 
   The privacy guarantee of Algorithm~\ref{alg: fix_k} stated as follows.
   \begin{theorem}\label{thm: t_composition}
   Algorithm~\ref{alg: fix_k} obeys $\delta_t$-approximately $\rho$-zCDP.
   	\end{theorem}
\begin{proof}
	The proof applies the composition theorem to bound the total RDP using the chosen approximate zCDP parameter of Algorithm~\ref{alg: main_ada} and the zCDP of the possible invocation of the Top-$k$ EM.
\end{proof}
  
  In terms of utility, this algorithm is never worse by more than a factor of $2$ than using all budget for Top-$k$ EM. However, if there is a large gap at the position $\tilde{k}$, the analyst only has to pay half of the total budget  to release the top $\tilde{k}$ set and use the noise scale $\frac{\sqrt{2\min\{|\tilde{k}-k|, k\}}}{\epsilon_{EM}}$ to select the remaining $|\tilde{k}-k|$ candidates. 

\subsection{Application to model agnostic learning with multi-label classification}\label{sec: pate}
\vspace{-2mm}
 A direct application of our private top-$k$ selection algorithm is in the \textit{private model agnostic learning}~\citep{bassily2018model} (a.k.a. the private knowledge transfer model). \textit{Private model agnostic learning} is a promising recent advance for differentially private deep learning that can avoid the explicit dimension dependence of the model itself and substantially improve the privacy-utility trade-offs. This framework requires an unlabeled public dataset to be available in the clear. 

The Private Aggregation of Teacher Ensembles (PATE)~\citep{papernot2017,papernot2018} is the main workhorse to make this framework being practical. 
PATE first randomly partition the private dataset into $T$ splits and trains a teacher model on each split. Then an ensemble of teacher models make predictions on unlabeled public data, and their sanitized majority votes are released as pseudo-labels.  Lastly, a student model is trained using pseudo-labeled data and is released to the public. 
The privacy analysis of PATE can be thought of as a tight composition over a sequence of private queries via RDP, where each query applies a Gaussian mechanism to releases the top-$1$ label. 

\begin{example}[PATE with multi-class classification tasks \citep{papernot2018}]
For each unlabeled data $x$ from the public domain, let $f_j(x) \in [c]$ denote the $j$-th teacher model's prediction and $n_i$ denotes the vote count for the $i$-th class (i.e., $n_i:= \sum_j |f_j(x)=i|$). PATE framework labels $x$ by $\cM_{PATE}(x) = argmax_i(n_i(x) + \cN(0, \sigma^2))$.  $\cM_{PATE}$ guarantees $(\alpha, \alpha/\sigma^2)$-RDP for each labeling query.
\end{example}
Unfortunately, the current PATE framework only supports the multi-class classification tasks instead of the generalized multi-label classification tasks, while the latter plays an essential role in private language model training (e.g., tag classification). The reasons are two-fold: first, the label space is large and each teacher in principle could vote for all labels (i.e., the global sensitivity grows linearly with the label space), thus preventing a practical privacy-utility tradeoff using Gaussian mechanism. 
Secondly, previous private top-$k$ selection algorithms do not target multiple releases of private queries; thus, there is a lack of a tight private accountant. Our algorithm naturally narrows this gap by providing an end-to-end RDP framework that enables a sharper composition. Moreover, an adaptively chosen $k$ is indeed favorable by PATE, as the number of ground-truth labels can be different across different unlabeled data. We provide one example of applying Algorithm~\ref{alg: main_ada} to solve multi-label classification tasks.

\begin{example}[PATE with multi-label classification tasks]
	For each unlabeled data $x$ from the public domain, let $f_j(x) \in \{0, 1\}^c$ denote the $j$-th teacher model's prediction and $n_i$ denotes the vote count for the $i$-th class (i.e., $n_i:= \sum_j |f_{j,i}(x)|$). PATE framework labels $x$ by $\cM_{PATE}(x)$=Algorithm~\ref{alg: main_ada}.  $\cM_{PATE}$ answers $T$ labeling queries guarantees $T\delta_t$-approximated-$T \rho$-zCDP.
\end{example}



%% file: experiment.tex
\section{Experiment}

\noindent\textbf{EXP1: Evaluations of k-selection with a fixed $k$.}
In Exp 1, we compare our \textsc{StableTopK} with recent advances (TS\citep{carvalho2020differentially} and the Limited Domain(LD)~\citep{durfee2019}) for private top-$k$ section algorithms. 
We replicate the experimental setups from \citep{carvalho2020differentially}, which contains two location-based check-ins datasets Foursquare~\citep{yang2014modeling} and BrightKite~\citep{cho2011friendship}.  
BrightKite contains over $100000$ users and $1280000$ candidates.  Foursquare contains $2293$ users with over $100000$ candidates. We assume each user gives at most one count when she visited a certain location. The goal is to select the top-$k$ most visited locations, where $k$ is chosen from $\{3, 10 ,50\}$.   

\noindent\textbf{Comparison Metrics and Settings} Similar to \citep{carvalho2020differentially}, we consider the \textit{proportion of true top-$k$} metric, which evaluates the number of true top-$k$ elements returned divided by $k$. For privacy budgets, we set $\delta = 1/n$ and consider $\epsilon$  being chosen from $\{0.4, 0.8, 1.0\}$. 
In the calibration, we  split half of $\delta$ as the failure probability $\delta_t$. Then, we use the RDP to $(\epsilon, \delta)$-DP conversion rule to calibrate $\sigma$ using the remaining privacy budget $(\epsilon, \delta - \delta_t)$.  For simplicity, we use $r= 0$ in \textsc{StableTopK}.
 For TS and LD, we report their results from \cite{carvalho2020differentially}. 

\noindent\textbf{Observation} By increasing privacy budget from $\epsilon=0.4$ to $\epsilon=1.0$,  the ``accuray'' increases for all algorithms. Moreover, our  \textsc{StableTopK} consistently outperforms TS, LD in the specific settings we consider.

\begin{figure}[tbh]
	\centering	
	\subfigure[Covid-19 dataset with a small $k$ \label{fig: exp_covid}]{
		\includegraphics[width=0.52\textwidth]{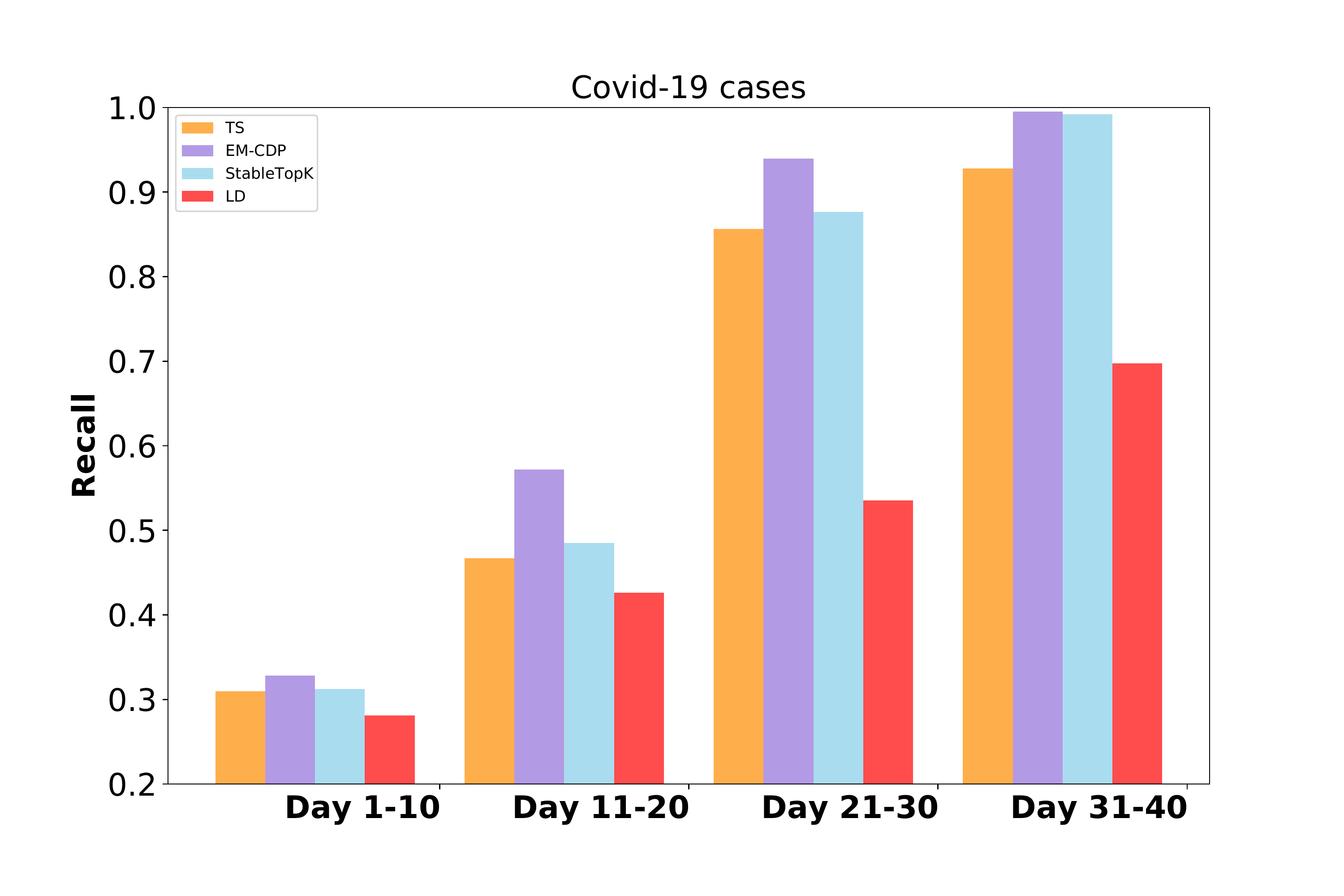}}
	\subfigure[Synthetic dataset \label{fig: sync}]{
		\includegraphics[width=0.45\textwidth]{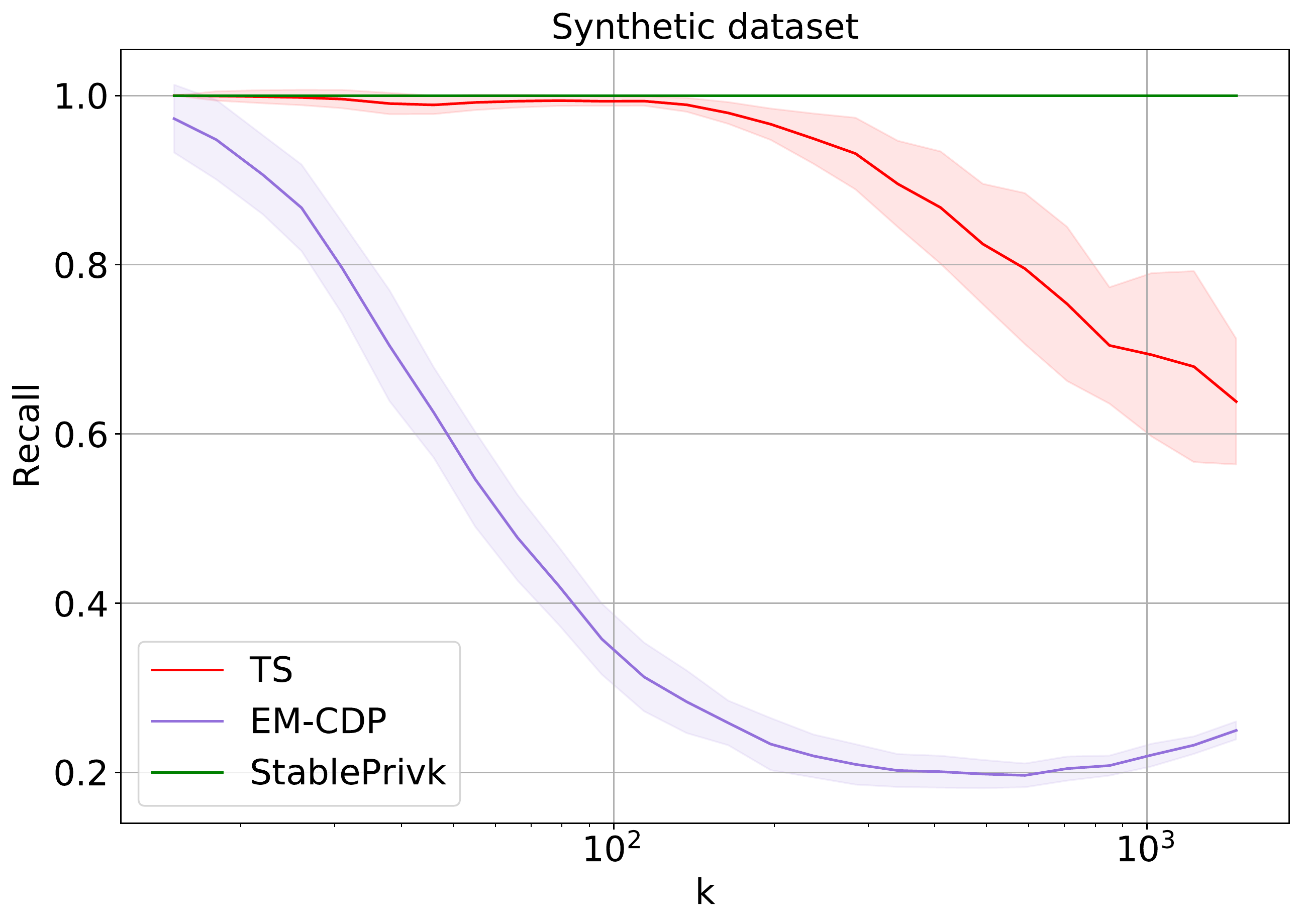}}
	\caption{ Figure~\ref{fig: exp_covid} evaluates composed top-$k$ selection with varied data distribution. Figure~\ref{fig: sync} compares top-$k$ selection with different choice on $k$. }
\end{figure}

\begin{table*}
	\centering
	\vspace{-4mm}
\centerline{
	\resizebox{\columnwidth}{!}{
		\tabcolsep=0.001cm
	\begin{tabular}{c@{\hskip 5mm}c@{\hskip 5mm}|c@{\hskip 6mm}c@{\hskip 6mm}c@{\hskip 6mm}|c@{\hskip 6mm}c@{\hskip 6mm}c@{\hskip 6mm}|c@{\hskip 5mm}c@{\hskip 5mm}c}
		\toprule
		Datasets & Methods & \multicolumn{3}{c}{$\epsilon=0.4$}&   \multicolumn{3}{|c}{$\epsilon=0.8$} &   \multicolumn{3}{|c}{$\epsilon=1.0$} \\
		\multirow{5}{*}{BrightKite}  &-&k:3&10&50&k:3&10&50&k:3&10& 50\\
		\hline
		&TS  &1.00&0.77&0.14&1.00&0.80&0.16&1.00&0.80&0.18\\
		    &LD  &1.00&0.47&0.10&1.00&0.79&0.25&1.00&0.88&0.28 \\
		 &  \textsc{StableTopK}&1.00&\textbf{0.91}&\textbf{0.61}&\textbf{1.00}&\textbf{1.00}&0.67&1.00&\textbf{1.00}&\textbf{0.67}\\
		\hline
			\multirow{5}{*}{Foursquare}  &TS  &\textbf{1.00}&0.64&0.11&\textbf{1.00}&0.90&0.18&\textbf{1.00}&0.90&0.18\\
			&LD  &0.68&0.62&0.13&0.75&0.85&0.24&0.91&0.94&0.28 \\
			& \textsc{StableTopK} &\textbf{1.00}&\textbf{0.72}&\textbf{0.48}&\textbf{1.00}&\textbf{1.00}&\textbf{0.67}&\textbf{1.00}&\textbf{1.00}&\textbf{0.67}\\
				\hline
	\end{tabular}}}
	\caption{EXP2: Comparison of top-$k$ selection with a fixed $k$}
	\label{tab: exp2}
	\bigskip
		\centerline{
	  \resizebox{ 0.8\columnwidth}{!}{
		\tabcolsep=0.001cm
		\begin{tabular}{c@{\hskip 5mm}|c@{\hskip 5mm}|c@{\hskip 5mm}|c@{\hskip 5mm}|c}
			\toprule
			Datasets & Methods &  $\epsilon$& Accuracy&   Non-Private  Accuracy \\
			\hline
			\multirow{3}{*}{CelebA}  &  PATE& >10&	$85.0 \pm 0.1\%$ &	\multirow{3}{*}{$89.5\pm 0\%$} \\
			& PATE-$\tau$&7.7 &$85.1\pm 0.2\%$& \\
			& StableTopK& 3.6&$85.0 \pm 0.2\%$&\\
			\hline
		\end{tabular}}
	}
	\caption{EXP3: Evaluations on CelebA datasets with $\delta=10^{-6}$.}
	\label{tab: celeba}
\end{table*}
\noindent\textbf{EXP2: Multiple top-$k$ queries}
The behaviors of top-$k$ mechanisms can be varied for different data distribution, $k$ and privacy budgets.
  To study their behaviors, we  design two groups of experiments  --- one with a fixed $k$ but various data distribution and another with a range of $k$.

We first consider the case when $k$ is fixed with an instantiation in releasing daily top-$k$ states that has the largest Covid-19 cases. 
We will use the United States Covid-19 Cases by State from 2020-03-12 to 2020-05-12 and assume one person can contribute at most one case on the daily case report.    

\noindent\textbf{Baselines and Metrics} TS and LD are two baselines.  As the CDP/RDP analysis of both TS and LD is unknown, we use advanced composition to allocate the privacy budget $(\epsilon, \delta)$ over $T$ queries.  

Another baseline we will use is the exponential mechanism EM-CDP. The exponential mechanism admits a tighter CDP analysis due to its bounded range property. We will add Gumbel noise to each count and report the indices with the top-$k$ highest noisy counts. 
In the experiment, we average the recall of the top-$k$ set over a fixed time interval (e.g., $T=10$ days) and repeat each experiment for $100$ trials.

In Figure~\ref{fig: exp_covid},  we consider $k=15$ and $(0.1, 10^{-6})$-DP instances of EM-CDP, TS, LD and our StableTopK(fixed K)  For each mechanism, we first calibrate their noise scale such that the composition over $10$ days satisfy $(0.1, 10^{-6})$-DP.  We then simulate four groups of the time interval: Day 1-10, Day 11-20, Day 21-30 and Day 31-40  such that $\#$ composition is the same  but the distribution of daily covid-19 cases is varied. Note that there was a exponential growth on the covid-19 cases between 2020-03-12 to 2020-05-12, which leads to an increasing gap  between the $k$-th and the $k+1$-th count.

EM-CDP performs best in all time intervals, especially when there are small gaps between the vote counts (see Day 1-10 and Day 11-20). When the gap is small, both TS and StableTopK will likely fail on the stability test, which will result in a substitute of the exponential mechanism using half of the privacy budget. Therefore, both TS and StableTopk perform worse than EM-CDP. The number of indices returned by LD can be smaller than $k$, especially when there is no large gap among vote counts. Thus it obtains the worst recall rate over all intervals.
When the gap is large, all mechanisms achieve better performance. StableTopk is still slightly worse than EM-CDP on Day 31-40 though the latter requires splitting the privacy budget into $k$ pieces. We conjecture this is because the $k$ we use is small, which diminished the effect of  ``unavoidable $O(\sqrt{k})$ dependence in $\epsilon$'' in EM-CDP.

Therefore, we next construct a synthetic example to investigate the effect on $k$.  The synthetic histogram has $15000$ bins,  where all top $k$ bins have  $700$ counts, and the remaining $15000-k$ bins have $0$ counts.  We range $k$ from $10$ to $1500$ with $(0.15, 10^{-6})$-DP instances of EM-CDP, TS and our StableTopK.
The line in Figure~\ref{fig: sync} plots the mean recall rate (of answering one-time top-k query) from $100$ trials, and the shaded region spans with the standard deviation for each mechanism.  StableTopK outperforms all mechanisms, especially when $k$ is large.  This is because the utility of StableTopK is determined by the gap at the $k$-th position rather than how large a k is.   StableTopK is clearly better than TS when $k$ is large as the latter needs to test the distance to instability for at most $k$ positions. We note that~\citet{lyu2017understanding} has a similar observation --- EM outperfoms SVT in  the non-interactive setting.
 Though EM-CDP admits a tight composition through CDP, its peeling procedure requires splitting its privacy budget into $k$ splits for each subroutine.   Therefore, EM-CDP is worse than \textsc{StableTopK} when $k$ is sufficiently large. 
%
%

\noindent\textbf{EXP3: Evaluation with multi-label classification tasks.}
\textbf{CelebA}~\citep{liu2015deep} is a large-scale face attribute dataset with $220k$ celebrity images, each with $40$ attribute annotations.
To instantiate the PATE framework, we take the original training set as the private domain and split it into $800$ teachers. Similar to the implementation from~\citet{private_kNN}, we randomly pick $600$ testing data to simulate unlabeled public data and using the remaining data for testing. We train each teacher model via a Resnet50m structure~\citep{he2016deep}.  As there is no strict restriction on an exact $k$ output, we apply a Gaussian variant of Algorithm~\ref{alg: main_ada} (i.e., replace the second step in Algorithm~\ref{alg: main_ada} with RNM-Gaussian) with noisy parameters $\delta_t=10^{-9}, \sigma_1=50, \sigma=60$. $\sigma_1$ is used in RNM-Gaussian.  
 Our result is compared to two baselines: PATE~\citep{papernot2018} and PATE-$\tau$~\citep{private_kNN}.  In PATE, the global sensitivity is $40$, as each teacher can vote for all attributes. To limit the global sensitivity, PATE-$\tau$ applies a $\tau$-approximation by restricting each teacher's vote that no more than $\tau$ attributes or contributions will be averaged to $\tau$. We remark that though the $\tau$ approximation approach significantly reduces the global sensitivity, the choice on $\tau$ shall not be data-dependent. 
In Table~\ref{tab: celeba}, we align the accuracy of three DP approaches and compare their accuracy at the test set.  We report the privacy cost based on the composition of over $600$ labeling queries from the public domain.  For StableTopK,  the reported $\epsilon$ is based on the RDP to DP conversion rule using $\tilde{\delta} = 10^{-6} - 600\times 10^{-9}$.
Each experiment is repeated five times. Our StableTopK (adaptive K) algorithm saves half of the privacy cost compared to PATE-$\tau$ while maintaining the same accuracy.


%% file: new_appendix.tex
\vspace{-4mm}
\section*{Organization of the Appendix}

In the appendix, we first state the composition and conversion rules in Sec~\ref{sec: rule}. In Sec~\ref{sec:ptr-laplace} we provide the description and the analysis of PTR-with Laplace-mechanism-based tests and Gaussian mechanism-based tests. Finally, Sec~\ref{sec: proof} provides all other proofs that were  omitted in the main paper, including that for the generalized (and Gaussian) RNM.


\section{Conversion and composition rules for approximated RDP}\label{sec: rule}
Recall our definition of approximate RDP.
\begin{definition}[Approximate Renyi Differential Privacy]
	We say a randomized algorithm $\cM$ is $\delta$-approximate-$(\alpha, \epsilon_\cM(\alpha))$-RDP with order $\alpha \geq 1$, if for all neighboring dataset $D$ and $D'$, there exist events $E$ (depending on $\cM(D)$ )and $E'$ (depending on $\cM(D')$) such that $\pr[E]\geq 1-\delta$ and $\pr[E']\geq 1-\delta$, and $\forall \alpha\geq 1$, we have
	$	\mathbb{D}_{\alpha}(\cM(D)|E||   \cM(D')|E')\leq  \epsilon_\cM(\alpha)$.
\end{definition}
When $\delta$ is 0, $0$-approximate-RDP is RDP.   Similar to \citep{bun2016concentrated}, the approximate-RDP satisfies the composition and post-processing property. 
\begin{lemma}[Composition rule]
	Let $\cM_1$ satisfies $\delta_1$-approximate-$(\alpha, \epsilon_{\cM_1}(\alpha))$-RDP and  $\cM_2$ satisfies $\delta_2$-approximate-$(\alpha, \epsilon_{\cM_2}(\alpha))$-RDP. Then the composition of $\cM_1$ and $\cM_2$ satisfies 
	$(\delta_1+\delta_2)$-approximate-$(\alpha, \epsilon_{\cM_1}(\alpha)+ \epsilon_{\cM_2}(\alpha))$-RDP.
\end{lemma}
\begin{lemma}[Conversion rule]
	Let $\cM$ satisfies $\delta_1$-approximate-$(\alpha, \epsilon_{\cM}(\alpha))$-RDP. Then it also satifies $(\epsilon_\cM(\alpha)+ \frac{\log(1/\delta)}{\alpha-1}, \delta+\delta_1)$-DP.
	\begin{proof}
		$\cM$ satisfies $\delta_1$-approximate-$(\alpha, \epsilon_{\cM}(\alpha))$-RDP implies that there exists an pairing event $E$ and $E'$ such that  $\mathbb{D}_\alpha (\cM(D)|E || \cM(D')|E')\leq \epsilon_{\cM}(\alpha) $ and $\pr[E]\geq 1-\delta_1, \pr[E']\geq 1-\delta_1$.
		Condition on $E$ and $E'$, we apply the RDP conversion rule~\citep{mironov2017renyi}, which gives us $(\frac{\log(1/\delta)}{\alpha-1} + \epsilon_\cM(\alpha))$-DP. Then we combine the failure probability  $\delta_1$ and $\delta$, which completes the proof.
	\end{proof}
\end{lemma}

\section{Propose-Test-Release with Gaussian and Laplace noise}\label{sec:ptr-laplace}

Instead, the PTR framework is less restrictive than the smooth sensitivity.
This approach first proposes a good estimate of the local sensitivity and then testing whether this is a valid upper bound. If the test passes, we then calibrate the noise according to the proposed test.  If the test instead failed, the algorithm stops and returns ``no-reply''.

\begin{lemma}\label{lem: ptr}
	Let $\hat{q}_k$ be a private release of $q_k$ that obeys $(\alpha, \epsilon_{gap}(\alpha))$-RDP and $\pr[\hat{q}_k \geq q_k]\leq \delta_t$ (where the probability is only over the randomness in releasing $\hat{q}_k$). If $\hat{q}_k$ passes the threshold check (i.e., $\hat{q}_k \geq 1$), the algorithm releases the set of top-$k$ indices directly satisfes $\delta_t$-approximately-$(\alpha, \epsilon_{gap}(\alpha))$-RDP. 
\end{lemma}
In the proof, the local sensitivity depends on the private data only through the gap $q_k$. Thus we can construct a private lower bound of $q_k$ such that --- if the PTR test passes, then with probability at least $1-\delta_t$, the local sensitivity is 0. Therefore we do not need to randomize the output. If the PTR test fails, the algorithm is $\epsilon_{gap}(\alpha)$-RDP due to post-processing. 

\begin{remark}
	The bottleneck of PTR approaches is often the computation efficiency of bounding the local sensitivity. Our algorithm addresses this issue by exploiting the connection to $q_k$, which only takes $O(1)$ time to validate the local sensitivity. Moreover, most prior work on PTR approaches only accounts for approximate differential privacy. Our new RDP analysis enables PTR algorithms to permit tighter analyses of privacy loss over multiple releases of statistics. 	
\end{remark}

Our algorithm applies a variant of the PTR framework, which first constructs a high-confidence private upper bound of the local sensitivity and then calibrates the noise accordingly. We formalize the idea in the following theorem.
\begin{algorithm}[t]
	\caption{Propose-test-release (PTR) with Gaussian Noise}
	\label{alg: ptr_gau}
	\begin{algorithmic}[1]
		\STATE{ \textbf{Input} Histogram $h$, noise parameter $ \sigma_2$ and the privacy parameter $\delta_t$}
		\STATE  Let $i_{(1)}, ..., i_{(k)}$ be the unordered indices of the sorted histogram.
		\STATE{Set the gap} $q_k = h_{(k)} -  h_{(k+1)}$
		\STATE {Propose a private lower bound of $q_k$:} $\hat{q}_k=  \max\{1, q_k\} + \cN(0, \sigma_2^2) -\sigma_2\sqrt{2\log(1/\delta_t)}$
		\STATE \textbf{If} $\hat{q}_k \leq 1$, \textbf{Return $\perp$}
		\STATE \textbf{Return $i_{(1)}, ... , i_{(k)}$} 
	\end{algorithmic}
\end{algorithm}

\begin{algorithm}[t]
	\caption{Propose-test-release (PTR) with Laplace Noise}
	\label{alg: ptr_lap}
	\begin{algorithmic}[1]
		\STATE{ \textbf{Input} Histogram $h$, noisy gap $\hat{q}_k$,  privacy parameter $\delta_t, \epsilon$}
		\STATE Let $i_{(1)}, ..., i_{(k)}$ be the unordered indices of the sorted histogram.
		\STATE{Set the gap} $q_k = h_{(k)} -  h_{(k+1)}$
		\STATE {Propose a private lower bound of $q_k$: $\hat{q}_k =q_k + \mathrm{Lap}(1/\epsilon) -\log(1/\delta_t)/\epsilon$} 
		\STATE \textbf{If} $\hat{q}_k \leq 1$, \textbf{Return $\perp$}
		\STATE \textbf{Return $i_{(1)}, ... , i_{(k)}$} 
	\end{algorithmic}
\end{algorithm}

Next, we work out the detailed calibration of PTR approaches using Laplace/Gaussian noise and provide their privacy guarantee in the following corollary. 
\begin{corollary}[Privacy guarantee of PTR variants]
	Algorithm~\ref{alg: ptr_lap} (PTR-Laplace) satisfies $(\epsilon, \delta_t)$-DP. Algorithm~\ref{alg: ptr_gau} (PTR-Gaussian)  satisfies $\delta_t$-approximately-$(\alpha, \frac{\alpha}{2\sigma_2^2})$-RDP.
\end{corollary}

The Laplace noise used in PTR-Laplace is heavy-tailed, which requires the threshold in Algorithm~\ref{alg: ptr_lap} to be  $O(\log(1/\delta_t))$ in order to control the failure probability being bounded by $\delta_t$. In contrast, Algorithm~\ref{alg: ptr_gau} with Gaussian noise requires a much smaller threshold --- $O(\sqrt{\log1/\delta_t})$ due to its more concentrated noise.

\begin{theorem}[Accuracy comparison]\label{thm: utility} 
	For one-time DP top-$k$ query, the minimum gap $h_{(k)}-h_{(k+1)}$ needed to output $k$ elements with probability at least $1-\beta $ is $h_{(k)} - h_{(k+1)}\geq 1 + (\log 1/\delta  + \log 1/\beta)/(\epsilon/4)$ for PTR-Gaussian while $h_{(k)}-h_{(k+1)} > 1+\log(1/\delta)/\epsilon +\log(1/\beta)/\epsilon$ for PTR-Laplace.
\end{theorem}
 \begin{proof}
 	With $q_k \geq  1+\log(1/\delta_t)/\epsilon + \log(1/\beta)/\epsilon$, we have 
 	\[\hat{q}_k \geq  1 + \log(1/\delta_t)/\epsilon + \log(1/\beta)/\epsilon + \mathrm{Lap}(1/\epsilon) - \log(1/\delta_t)/\epsilon =  \log(1/\beta)/\epsilon +1 + \mathrm{Lap}(1/\epsilon) \]
 	PTR-Laplace outputs $k$ elements only when $\hat{q}_k>1$.
 	Therefore, the failure probability is bounded by $\pr[\mathrm{Lap}(1/\epsilon)> \log(1/\beta)] =\beta$.
 \end{proof}

PTR-Laplace outperforms PTR-Gaussian for one-time query as we are using a loose calibration $\sigma_2 = \frac{\sqrt{2\log(1.25/\delta)}}{\epsilon}$. However, if we align the zCDP parameter (e.g., $\epsilon_{gap}(\alpha)=\frac{\alpha \epsilon^2}{2}$), then $\sigma_2 = 1/\epsilon$ and gives us the minimum gap to be
$1 +\frac{1}{\epsilon}\sqrt{2\log(1/\delta)}+\frac{1}{\epsilon}\sqrt{2\log(1/\beta)}$ for PTR-Gaussian. This explains why the Gaussian version of PTR is superior under composition.

\section{Omitted Proofs}\label{sec: proof}
\begin{theorem}[Restatement of Theorem~\ref{thm: rnm_k}]
	Let $\cM_g$ denote any noise-adding mechanism that satisfies $\epsilon_g(\alpha)$-RDP for a scalar function $f$ with global sensitivity $2$.  
	Assume Report-Noisy-Max adds the same magnitude of noise to each coordinate, then the algorithm obeys
	$\epsilon_\alpha(\cM(D)||\cM(D')) \leq \epsilon_g(\alpha) +\frac{\log m}{\alpha-1}$
\end{theorem}
\begin{proof}
	We use $i$ to denote any possible output of the Report-Noisy-Max $\cM(D)$. The  Report-Noisy-Max aims to select an coordinate $i$ that maximizes $C_i$ in a privacy-preserving way, where $C_i$ denotes the difference between $h_{(i)}(D)$ and $h_{(i+1)}(D)$.  Let $C'$ denote the vector of the difference when the database is $D'$. We will use the Lipschitz property: for all $j \in [m-1]$, $1+C'_j \geq C_j$. This is because adding/removing one data point could at most change $C_j$ by $1$ for $\forall j\in[m-1]$.
		Throughout the proof, we will use $p(r_i), p(r_j)$ to denote the pdf of $r_i$ and $r_j$, where $r_i$ denote the realized noise added to the $i$-th coordinate.
	
	From the definition of Renyi DP, we have
	\begin{align}
		\mathbb{D}_\alpha(\cM(D)||\cM(D'))=\frac{1}{\alpha -1}\log \mathbb{E}_{i \sim D'}\bigg[ \frac{\pr[\cM(D)=i]^\alpha}{\pr[\cM(D')=i]^\alpha}\bigg] = \frac{1}{\alpha -1}\log \sum_{i=1}^m \frac{\pr[\cM(D)=i]^\alpha}{\pr[\cM(D')=i]^{\alpha-1}}\label{equation: rdp_def}
	\end{align}
	Our goal is to upper bound $(*)= \sum_{i=1}^m \frac{\pr[\cM(D)=i]^\alpha}{\pr[\cM(D')=i]^{\alpha-1}} $.
	The probability of outputting $i$ can be written explicitly as follows:
	\begin{align*}
		&\pr[\cM(D) = i] = \int_{-\infty}^\infty p(r_i) \pr[C_i + r_i > \max_{j \in [m], j \neq i} \{C_j + r_j\}] dr_i\\
		&=\int_{-\infty }^\infty p(r_i -2)  \pr[C_i + r_i -2 > \max_{j \in [m], j \neq i} \{C_j + r_j\}] dr_i\\
		&=\int_{-\infty }^\infty p(r_i) \bigg(\frac{p(r_i -2)}{p(r_i)} \bigg) \pr[C_i + r_i -2 > \max_{j \in [m], j \neq i} \{C_j + r_j\}] dr_i\\
		&= \mathbb{E}_{r_i} \bigg[ \bigg(\frac{p(r_i -2)}{p(r_i)}\bigg) \pr[C_i + r_i -2 > \max_{j \in [m], j \neq i} \{C_j + r_j\}] \bigg]  
	\end{align*}
	In the first step, the probability of $\pr[C_i + r_i > \max_{j \in [m], j \neq i} \{C_j + r_j\}]$ is over the randomness in $r_j$.
	Substituting the above expression to the definition of RDP and apply Jensen's inequality
	\begin{align*}
		&(*)= \sum_{i=1}^m \frac{ \bigg[\mathbb{E}_{r_i}  \bigg(\frac{p(r_i -2)}{p(r_i)}\bigg) \pr[C_i + r_i -2 > \max_{j \in [m], j \neq i} \{C_j + r_j\}] \bigg]^\alpha}{ \bigg[  \mathbb{E}_{r_i}  \pr[C'_i + r_i  > \max_{j \in [m], j \neq i} \{C'_j + r_j\}] \bigg]^{\alpha-1}}\\
		&\leq \sum_{i=1}^m\mathbb{E}_{r_i}  \bigg(\frac{p(r_i -2)}{p(r_i)}\bigg)^\alpha \bigg(\frac{ \pr[C_i + r_i -2 > \max_{j \in [m], j \neq i} \{C_j + r_j\}] }{ \pr[C'_i + r_i  > \max_{j \in [m], j \neq i} \{C'_j + r_j\}] } \bigg)^{\alpha-1} \cdot \pr[C_i + r_i -2 > \max_{j \in [m], j \neq i} \{C_j + r_j\}]
	\end{align*}
	We apply Jensen's inequality to bivariate function $f(x,y) = x^\alpha y^{1-\alpha}$, which is jointly convex on $\cR^2_+$ for $\alpha \in (1, +\infty)$.
	The key of the analysis relying on bounding $ (**)=\bigg(\frac{ \pr[C_i + r_i -2 > \max_{j \in [m], j \neq i} \{C_j + r_j\}] }{ \pr[C'_i + r_i  > \max_{j \in [m], j \neq i} \{C'_j + r_j\}] } \bigg)$.
	Note that $D'$ is constructed by adding or removing one user's all predictions from $D'$.
	In the worst-case scenario, we have $C'_j = C_j +1$ for every $j \in[m], j \neq i, C_i = C'_i+1$ . Based on the Lipschitz property, we have
	\[ \pr[C'_i +r_i  > \max_{j \in [m], j \neq i} \{C'_j + r_j\}] \geq  \pr[C_i + r_i -2> \max_{j \in [m], j \neq i} \{C_j + r_j \}]  \]
	which implies $(**)\leq 1$.
	Therefore, we have 
	\[\epsilon_\cM(\alpha) \leq \frac{1}{\alpha-1} \log \sum_{i=1}^ m \mathbb{E}_{r_i}  \bigg(\frac{p(r_i -2)}{p(r_i)}\bigg)^\alpha \leq \epsilon_g(\alpha) +\frac{\log(m)}{\alpha-1} . \]
\end{proof}

\begin{corollary}[Restatement of Corollary~\ref{coro_gau}]
	RNM-Gaussian (the second line in Algorithm~\ref{alg: main_ada}) with Gaussian noise $\cN(0, \sigma_1^2)$ satisfies $(\frac{2\alpha}{\sigma_1^2}+\frac{\log m}{\alpha -1})$-RDP. 
\end{corollary}
\begin{proof}
	For a function $f: \cD\to \cR$ with L2 sensitivity $2$,the RDP of Gaussian mechanism with Gaussian noise $\cN(0, \sigma_1^2)$ satisfies $(\alpha, \frac{2\alpha}{\sigma_1^2})$-RDP.  We complete the proof by plugging in $\epsilon_g(\alpha )= \frac{2\alpha}{\sigma_1^2}$ into Theorem~\ref{thm: rnm_k}.
\end{proof}

\begin{lemma}[Restatement of Lemma~\ref{lem: ptr}]
	Let $\hat{q}_k$ obeys $\epsilon_{gap}(\alpha)$-RDP and $\pr[\hat{q}_k \geq q_k]\leq \delta_t$ (where the probability is only over the randomness in releasing $\hat{q}_k$). If $\hat{q}_k$ passes the threshold check
	, the algorithm releases the set of top-$k$ indices directly satisfies $\delta_t$-approximately-$(\alpha, \epsilon_{gap}(\alpha))$-RDP. 
\end{lemma}

\begin{proof}
	We start with the proof for $\delta_t$-approximately-$(\alpha, \epsilon(\alpha))$-RDP. 
	Denote $\cM_{1}$ be the mechanism that releases the set of top-$k$ indices directly (without adding noise) if $\hat{q}_k$ passes the threshold check ($\hat{q}_k>1$).
	
	Then let us discuss the two cases of the neighboring pairs $D,D'$.
	\begin{enumerate}[label=(\alph*)]
		\item For neighboring datasets $D,D'$ where the Top-$k$ indices are the same, the possible outputs are therefore $\{\perp, \mathrm{Top-k}(D)\}$ for both $\cM_1(D),\cM_1(D')$. Notice that $|q_k(D)-q_k(D')| \leq 1$, thus in this case 
		$$
		\mathbb{D}_\alpha(\cM_1(D) \| \cM_1(D'))= D_\alpha(\mathbf{1}(\hat{q}_k(D)>1) \| \mathbf{1}(\hat{q}_k(D')>1) ) \leq \mathbb{D}_\alpha(\hat{q}_k(D)\|\hat{q}_k(D') ) \leq \epsilon_{gap}(\alpha),	$$
		where the inequality follows from the information-processing inequality of the Renyi Divergence.
		Thus it trivially satisfies $\delta$-approximated-$(\alpha,\epsilon_{gap}(\alpha))$-RDP when we set $E$ to be the full set, i.e., $\Pr[E]=1 \geq 1-\delta$.
		\item For $D,D'$ where the Top-$k$ indices are different, then it implies that $q_k(D)\leq 1$ and $q_k(D')\leq 1$. In this case, we can construct $E$ to be the event where $\hat{q}_k\leq q_k$, i.e., the high-probability lower bound of $q_k$ is valid. Check that $\P[E]\geq 1-\delta$ for any input dataset.  Conditioning on $E$,  $\hat{q}_k \leq q_k \leq 1$ for both $D,D'$, which implies that $\Pr[\cM_1(D) = \perp | E] = \Pr[\cM_1(D') = \perp | E] = 1$. Thus, trivially $\mathbb{D}_{\alpha}(\cM(D) | E(D) \| \cM(D') | E(D') ) =0$ for all $\alpha$. For this reason, it  satisfies $\delta$-approximated-$(\alpha,\epsilon(\alpha))$-RDP for any function $\epsilon(\alpha)\geq 0$, which we instantiate it to be $\epsilon_{gap}(\alpha)$.
	\end{enumerate}

\end{proof}